%% file: Arxiv2022RLv1Oct.tex
\documentclass{article}
\input{preamble}

 \PassOptionsToPackage{numbers, compress}{natbib}

\newcommand{\bx}{\mathbf{x}}

\newcommand{\by}{\mathbf{y}}

\newcommand{\calF}{\mathcal{F}}

\newcommand{\bbR}{\mathbb{R}} 
\newcommand{\bbZ}{\mathbb{Z}} 
\newcommand{\bbE}{\mathbb{E}} 

     \usepackage[preprint]{neurips_2022}

\usepackage[utf8]{inputenc} 
\usepackage[T1]{fontenc}    
\usepackage{hyperref}       
\usepackage{url}            
\usepackage{booktabs}       
\usepackage{amsfonts}       
\usepackage{nicefrac}       
\usepackage{microtype}      
\usepackage{xcolor}         

\title{Generative Adversarial Nets:\\
	Can we generate a new dataset based on only one training set?}

\author{%
	Lan V. Truong\thanks{Use footnote for providing further information
		about author (webpage, alternative address)---\emph{not} for acknowledging
		funding agencies.} \\
	Department of Engineering\\
	University of Cambridge\\
	Cambridge, CB2 1PZ \\
	\texttt{lt407@cam.ac.uk} \\
}
\begin{document}
	\maketitle
	\begin{abstract}
		A generative adversarial network (GAN) is a class of machine learning frameworks designed by Goodfellow et al. in 2014. In the GAN framework, the generative model is pitted against an adversary: a discriminative model that learns to
		determine whether a sample is from the model distribution or the data distribution. GAN generates new samples from the same distribution as the training set. In this work, we aim to generate a new dataset that has a different distribution from
		the training set. In addition, the Jensen-Shannon divergence between the distributions of the generative and training datasets can be controlled by
		some target $\delta \in [0, 1]$. Our work is motivated by applications in generating new kinds of rices which have similar characteristics as a good rice. 
	\end{abstract}

	\section{INTRODUCTION} \label{sec:intro}
	Representation learning is a set of techniques that allows a system to automatically discover the representations from raw data needed for feature detection or classification from raw data. This replaces manual feature engineering and allows a machine to both learn the features and use them to perform a specific task. Feature learning can be either supervised or unsupervised. In supervised feature learning, features are learned using labeled input data. Examples include supervised neural networks, multilayer perceptron and (supervised) dictionary learning. In unsupervised feature learning, features are learned with unlabeled input data. Examples include dictionary learning, independent component analysis, autoencoders, matrix factorization and various forms of clustering.

	\subsection{Related Papers} 
	In the last few years, deep learning based generative models have gained more and more interest due to (and implying) some amazing improvements in the field. Relying on huge amount of data, well-designed networks architectures and smart training techniques, deep generative models have shown an incredible ability to produce highly realistic pieces of content of various kind, such as images, texts and sounds. Among these deep generative models, two major families stand out and deserve a special attention: Generative Adversarial Networks (GANs) \citep{Goodfellow2014GenerativeAN} and Variational Autoencoders (VAEs) \citep{Kingma2014AutoEncodingVB}.
	
	A variational autoencoder can be defined as being an autoencoder \citep{Kramer1991NonlinearPC} whose training is regularised to avoid overfitting and ensure that the latent space has good properties that enable generative process. Tolstikhin et al. proposed a Wasserstein Autoencoder (WAE), which minimizes a penalized form of the Wasserstein distance between the model distribution and the generative distribution \citep{TolBouGelSch18}. WAE shares many of the properties of VAEs such as stable training, encoder-decoder architecture, nice latent manifold structure while generating samples of better quality, as measured by the FID score.
	
	A generative adversarial network (GAN) is a class of machine learning frameworks designed by Goodfellow  et al. in  2014 \citep{Goodfellow2014GenerativeAN}.
	In GAN, the generative model learns to map from a latent space to a data distribution of interest, while the discriminative model distinguishes candidates produced by the generator from the true data distribution. The generative network's training objective is to increase the error rate of the discriminative network. Generative adversarial networks have applications in many fields such as fashion, art and advertising, science, video games, and audio synthesis. There is a veritable zoo of GAN variants. Conditional GANs \citep{Goodfellow2014GenerativeAN} are similar to standard GANs except they allow the model to conditionally generate samples based on additional information. For example, if we want to generate a cat face given a dog picture, we could use a conditional GAN. The GAN game is a general framework and can be run with any reasonable parametrization of the generator $G$ and discriminator $D$. In the original paper, the authors demonstrated it using multilayer perceptron networks and convolutional neural networks. Many alternative architectures have been tried such as Deep convolutional GAN \citep{Radford2016UnsupervisedRL}, Self-attention GAN \citep{Jiang2021TransGANTP}, Flow-GAN \citep{Grover2018FlowGANCM}. 
	
	\subsection{Motivations}
	There were some new variants of GAN which allow the use of multiple data distributions and the generated ones such as the conditional GAN.  However, these new variants of GAN require least two different training sets to generate a new one. In many applications in practice, we would like to generate a new dataset which have the same characteristic as a reference one. In this work, we aim to develop a new variant of GAN which allows to perform this task. Our work is motivated by applications in generating new kinds of rices which have similar characteristics as a good rice. 
	
	More specifically, assume that we have $L$ datasets with unknown distribution $p_1,p_2,\cdots,p_L$ for some $L\geq 1$. We aim to generate a new dataset which has a different distribution from the training datasets. In addition, the Jensen-Shannon divergence between the distribution of the generative dataset and a mixture data distribution can be controlled, i.e.  $ \rm{JSD}(\sum_{l=1}^L \alpha_l p_l,p_g)\leq \delta$ for some given non-negative tuple $(\alpha_1,\alpha_2,\cdots,\alpha_L)$ satisfying $\sum_{i=1}^L \alpha_i=1$ and $\delta \in [0,1]$.  For $L=1$, our algorithm generates a new dataset such that the Jensen-Shannon divergence between the distributions of the generative and the training data is upper bounded by some target $\delta \in [0,1]$. 
	
	This additional ``controllable property" is very important in many applications. For example, we sometimes need to generate a new cat gender (images) which owns most properties as an old gender of cats. In many other applications, we may increase the number of new generated images by lessening the distance requirement between the distributions of  data and generated ones compared with GAN or conditional GANs. 
	\subsection{Contributions}
	Our main contributions include: 
	\begin{itemize}
		\item We develop a new technique which allows to control the total variation between the distribution of the random vectors $\bx$ and $\by$ where $\by=\bx+\bz$ and $\bz$ is a sparse random vector with fixed distribution. 
		\item We propose a mechanism to which allows to loosen Jensen-Shannon divergence between the distribution of the generated distribution and the data distribution in the Goodfellow et al's model \citep{Goodfellow2014GenerativeAN}.
		\item We extend this new model to allows the use of multiple data distributions as in the conditional GAN.
		\item We illustrate our ideas on datasets Cfar10 and Cfar100, and generate new datasets based on only one dataset or a mixture of these two datasets for different values of $\delta$.
	\end{itemize}  
	\subsection{Notations} Consider a measurable $(\Omega,\calF)$ and probability measures $P$ and $Q$ defined on $(\Omega,\calF)$. The total variation distance between $P$ and $Q$ is defined as
	\begin{align}
	\delta(P,Q)=\sup_{A \in \calF}\big|P(A)-Q(A)|. 
	\end{align}
	If $\Omega \subset \bbR^d$ for some $d \in \bbZ_+$, and $p,q$ are corresponding probability functions w.r.t. the Lebesgue (or counting) measure in $\bbR^d$, then the total variation can be represented as
	\begin{align}
	\rm{TV}(P,Q)=\frac{1}{2}\int_{\Omega} \big|p(x)-q(x)\big|dx. 
	\end{align}
	The Jensen-Shannon divergence can be defined as
	\begin{align}
	\rm{JSD}(P,Q)=\frac{1}{2}D\big(P\|M\big)+\frac{1}{2}D\big(Q\|M\big)
	\end{align} where
	\begin{align}
	M=\frac{1}{2}\big(P+Q\big). 
	\end{align}
	\section{THEORETICAL RESULTS}\label{sec:setting}
	The original GAN is defined as a game where the generative model $G$ is pitted against an adversary: a discriminative model $D$ that learns to determine whether a sample is from model distribution or the data one. In \citep{Goodfellow2014GenerativeAN}, $D$ and $G$ play the following two player minimax game with value function $V(G,D)$:
	\begin{align}
	&\min_G \max_D V(D,G)=\bbE_{\bx \sim p_{\rm{data}}(\bx)}[\log D(\bx)] \nn\\
	&\qquad \qquad + \bbE_{\bz \sim p_{\bz}(\bz)}[\log(1-D(G(\bz)))] \label{G2}.
	\end{align}
	Then, given a fixed $G$, the optimal discriminator $D$ \cite[Prep.~1]{Goodfellow2014GenerativeAN} is
	\begin{align}
	D_G^*(\bx)=\frac{p_{\rm{data}}(\bx)}{p_{\rm{data}}(\bx)+p_g(\bx)}.
	\end{align}
	Let
	\begin{align}
	C(G)=\max_{D}V(G,D).
	\end{align}
	Then, the minimax game in Eq.~\eqref{G2} can be reformulated as:
	\begin{align}
	C(G)&=\max_{D}V(G,D)\\
	&=\bbE_{\bx \sim p_{\rm{data}}}\bigg[\log \frac{ p_{\rm{data}}(\bx)}{p_{\rm{data}}(\bx)+p_g(\bx)}\bigg]\nn\\
	&\qquad+ \bbE_{\bx \sim p_g}\bigg[\log \frac{p_g(\bx)}{p_{\rm{data}}(\bx)+ p_g(\bx)}\bigg].
	\end{align}
	Then, the following result was proved.
	\begin{theorem}\citep{Goodfellow2014GenerativeAN} \label{thm0} The global minimum of the virtual training criterion $C(G)$ is achieved if and only if $p_g=p_{\rm{data}}$. At that point $C(G)$ achieves the value $-\log 4$. More specifically, $C(G)=-\log 4+2 \rm{JSD}(p_{\rm{data}}\|p_g)$ where $\rm{JSD}(p_{\rm{data}}\|p_g)$ is the Jensen-Shannon divergence between the data distribution and the generative one.
	\end{theorem}
	Now, one interesting question is how to constraint the total variations between $p_{\rm{data}}$ and $p_g$ such that
	\begin{align}
	\rm{TV}(p_{\rm{data}}\|p_g)\leq \delta \label{mage}
	\end{align} for some given $\delta \in [0,1]$. At the first sight, a change in the loss function may help. However, finding a loss function for this target looks very challenging since the learning algorithm does not known $p_{\rm{data}}$ and $p_g$. Fortunately, the following trick can help us to satisfy the constraint in \eqref{mage} without much effort. We can achieve this target by adding a random noise vector to the training data to form a new training set $\bx_i'=\bx_i+\bz_i$ for all $i \in [n]$ with distribution $p'_{\rm{data}}$, where $\bz_i \sim p_Z$. Then, the following result can be proved.
	\begin{theorem} \label{main} Let $p_{\rm{data}}$ be a distribution on $\bbR^d$. For any $\delta \in [0,1]$, there exists a distribution $p_Z$ in $\bbR^d$ such that
		\begin{align}
		\rm{TV}(p_{\rm{data}}\|p_g)\leq \delta \label{mag}.
		\end{align}
		More specifically, the class of distributions $p_Z(z)=(1-\delta) \delta(z)+\gamma g(z)$ for any distribution $g$ in $\bbR^d$ satisfies \eqref{mag}.
	\end{theorem}
	Theorem \ref{main} gives us a freedom to choose the distribution $g$ on $\bbR^d$ to generate new samples . In other words, the new generated samples are functions of $g$ and of distribution $p_g$ such that \eqref{mag} holds.   
	
	To show Theorem \ref{main}, we first show the following lemma, whose proof can be found in Appendix \ref{lem:aproof}.
	\begin{lemma} \label{lem:a}
		Let $X$ be a random variable in $\bbR^d$. First, we show that there exists a distribution $Z \in \bbR^d$ such that $Y=X+Z$ satisfies
		\begin{align}
		\rm{TV}(p_X,p_Y) \leq \gamma
		\end{align} for any distribution of $X$ and $\gamma \in \bbR_{\geq 0}$.
	\end{lemma}
	Now, let's return to prove Theorem \ref{main}.
	\begin{proof}[Proof of Theorem \ref{main}] By Lemma \ref{lem:a}, there exists a distribution $P_Z$ on $\bbR^d$ such that
		\begin{align}
		\rm{TV}(p_{\rm{data}},p_{\rm{data'}})\leq \delta \label{G3a}. 
		\end{align}
		Now, we assume that $D$ and $G$ play the following two player minimax game with value function $V'(G,D)$:
		\begin{align}
		&\min_G \max_D V'(D,G)=\bbE_{\bx \sim p'_{\rm{data}}(\bx)}[\log D(\bx)] \nn\\
		&\qquad \qquad + \bbE_{\bz \sim p_{\bz}(\bz)}[\log(1-D(G(\bz)))] \label{G3}.
		\end{align}
		Then, by Theorem \ref{thm0}, we have
		$C'(G):= \max_D V'(D,G)=-\log 4$ if $p'_{\rm{data}}=p_g$. Hence, from \eqref{G3a} and \eqref{G3}, the global minimum of the virtual training criterion $C(G')$ is achieved if and only if
		\begin{align}
		\rm{JSD}(p_{\rm{data}},p_g)\leq \delta.
		\end{align}
		This concludes our proof of Theorem \ref{main}.
	\end{proof}
	Then, we propose the following variant of the training algorithm \cite[Algorithm 1]{Goodfellow2014GenerativeAN} for this new setting. We would like to generalize the result of Theorem \ref{main} to $L$ data distributions $p_{\rm{data}}^{(1)}, p_{\rm{data}}^{(2)},\cdots, p_{\rm{data}}^{(L)}$ such that $\rm{TV}( p_{\rm{data}}^{(l)},p_g)\leq \delta_l$ for all $l \in [L]$ where $\delta_1,\delta_2,\cdots,\delta_L$ is a sequence of real numbers in $[0,1]$. 
	
	\begin{algorithm}
		\caption{A GAN Algorithm with Total Variant Constraints}\label{alg:cap}
		\begin{algorithmic}
			\Ensure Training samples are of the same shape. $\alpha_1,\alpha_2,\cdots,\alpha_L \in [0,1]$ such that $\sum_{l=1}^L \alpha_l=1$
			\State Choose a training batch-size
			\item  Choose a number of epochs in each batch $k$
			\State Choose a noise distribution $p_Z$
			\State Choose $L$ arbitrary noise distributions $g_1,g_2,\cdots,g_L$ on the same signal space as training samples
			\State $N_0= \lceil  n/\texttt{batch-size }\rceil$
			\State $r \gets 1$
			\While{$r \leq N_0$}
			\For{\texttt{$k$ steps}}
			\State Sample a minibatch of $n$ samples $\{\bz^{(1)},\bz^{(2)},\cdots, \bz^{(n)}\}$ from noise prior $p_Z$
			\For{\texttt{each $l$ in $[L]$}}
			\State Sample a minibatch of $n$ examples $\{\bx_l^{(1)},\bx_l^{(2)},\cdots, \bx_l^{(n)}\}$ from data distribution $p_{l,\rm{data}}(\bx)$
			\State Sample a Bernoulli random variable $B$ with expectation $\gamma_l$
			\If{$B==1$}
			\State Sample a minibatch of $n$ samples $\{\bz_0^{(1)},\bz_0^{(2)},\cdots, \bz_0^{(n)}\}$ from noise prior $g_l$
			\For{\texttt{each $i$ in $[n]$}}
			$\tilbx_l^{(i)}\gets \bx_l^{(i)}+ \bz_0^{(i)}$ 
			\EndFor
			\ElsIf{$B==0$}
			\For{\texttt{each $i$ in $[n]$}}
			$\tilbx_l^{(i)}\gets \bx_l^{(i)}$
			\EndFor
			\EndIf
			\State Update the discriminator by ascending it stochastic gradient:
			$$
			\Delta_{\theta_d} \frac{1}{n}\sum_{i=1}^n\bigg[\sum_{l=1}^L \alpha_l \log D(\tilbx^{(i)})
			$$
			$$
			\qquad \qquad  +\log\bigg(1-D\bigg(G\big(\bz^{(i)}\big)\bigg)\bigg)\bigg]
			$$
			\EndFor
			\EndFor
			\State Sample a minibatch of $n$ samples $\{\bz^{(1)},\bz^{(2)},\cdots, \bz^{(n)}\}$ from noise prior $p_Z$
			\State Update the generator by descending its stochastic gradient:
			$$
			\Delta_{\theta_g}\frac{1}{n}\sum_{i=1}^n \log\bigg(1-D\bigg(G\big(\bz^{(i)}\big)\bigg)\bigg)
			$$
			\EndWhile
			
			The gradient-based updates can use any standard gradient-based learning rule. We used Adam in our experiments
		\end{algorithmic}
	\end{algorithm}
	Now, we prove the following result related to the convergence of Algorithm 1.
	\begin{theorem}\label{conv:thm} If $G$ and $\{D_l\}_{l=1}^L$ have enough capacity, and at each step of Algorithm 1, the discriminator $D_l$ is allowed to reach its optimum given $G$, and $p_g$ is updated so as to improve the criterion
		\begin{align}
		&\sum_{l=1}^L\alpha_l \bigg(\bbE_{\bx_l \sim p'_{l,\rm{data}}}\big[\log D^*(\bx_l)\big]\nn\\
		&\qquad + \bbE_{\bx \sim p_g}\big[\log(1-D^*(\bx))\big]\bigg)
		\end{align}
		then it holds that
		\begin{align}
		\limsup_{n \to \infty}\rm{TV}\big(p_{\rm{mix}},p_g\big)\leq \delta,
		\end{align} where
		\begin{align}
		p_{\rm{mix}}:=\sum_{l=1}^L \alpha_l p_{l,\rm{data}}.
		\end{align}
	\end{theorem}
	\begin{proof} The proof is based on \citep{Goodfellow2014GenerativeAN}. The training criterion for the discriminator $D$, given any generator $G$, is to maximize the quantity $V(G,D)$
		\begin{align}
		V(G,D)&=\sum_{l=1}^L\alpha_l \int_{\bx}  p'_{l,\rm{data}}(\bx)\log (D(\bx))d\bx\nn\\
		&\qquad  + \int_{\bx} p_{\rm{z}}(\bz)\log (1-D(G(\bz))) d\bz\\
		&=\int_{\bx}  \bigg(\sum_{l=1}^L \alpha_l p'_{l,\rm{data}}(\bx)\bigg) \log D(\bx)\nn\\
		&\qquad + p_g(\bx)\log(1-D(\bx))d\bx. 
		\end{align}
		For any $(a,b)\in \bbR^2 \setminus \{0,0\}$, the function $y\to a\log(y)+b\log(1-y)$ achieves its maximum in $[0,1]$ at $\frac{a}{a+b}$. 
		
		Hence, given $G$, the optimal discriminator $D$ is as follows:
		\begin{align}
		D_G^*(\bx)=\frac{\sum_{l=1}^L \alpha_l p'_{l,\rm{data}}(\bx)}{\sum_{l=1}^L \alpha_l p'_{l,\rm{data}}(\bx) + p_g(\bx)}.
		\end{align}
		
		Let $p'_{\rm{mix}}=\sum_{l=1}^L \alpha_l p'_{l,\rm{data}}$. Then, it follows that
		\begin{align}
		C(G)&=\max_D V(G,D)\\
		&=\bbE_{\bx \sim p'_{\rm{mix}}}[\log D_G^*(\bx)]+\bbE_{\bx \sim p_g}[\log (1-D_G^*(\bx))]\\
		&=-\log 4+ 2 \rm{JSD}\big(p'_{\rm{mix}}\|p_g\big), \label{A10} 
		\end{align} where \eqref{A10} follows from \cite[Eq.~(6)]{Goodfellow2014GenerativeAN}.	
		
		As each training step $l \in [L]$, it $D_l$ has enough capacity to reach its optimum given $G$, by \cite[Prep.~2]{Goodfellow2014GenerativeAN}, from \eqref{A10}, it holds that
		\begin{align}
		\rm{JSD}\big(p'_{\rm{mix}}, p_g\big)\to 0. \label{MAT0}
		\end{align}
		On the other hand, by the adding extra noise to the training set step, by Theorem \ref{main}, it holds that
		\begin{align}
		\rm{TV}\big(p_{l,\rm{data}},p'_{l,\rm{data}})\leq \delta, \qquad \forall l \in [L] \label{MAT1}.
		\end{align}
		This leads to
		\begin{align}
		&\rm{TV}\bigg(\sum_{l=1}^L \alpha_l p_{l,\rm{data}},\sum_{l=1}^L \alpha_l p'_{l,\rm{data}}\bigg) \nn\\
		&\qquad \leq \sum_{l=1}^L\alpha_l \rm{TV}(p_{l,\rm{data}}, p'_{l,\rm{data}}) \label{M1}\\
		&\qquad \leq \delta \bigg(\sum_{l=1}^L \alpha_l\bigg)\\
		&\qquad =\delta \label{M2},
		\end{align} where \eqref{M1} follows from the concavity of the total variation.
		
		From \eqref{M2}, we obtain
		\begin{align}
		\rm{JSD}\big(p_{\rm{mix}},p'_{\rm{mix}})\leq \delta \label{MAT2}
		\end{align} since Jensen-Shannon divergence is upper bounded by the total variation.
		
		Now, since the square root of
		the Jensen-Shannon divergence is a metric, it holds that
		\begin{align}
		&\sqrt{\rm{JSD}\big(p_{\rm{mix}},p_g)}\nn\\
		&\qquad \leq \sqrt{\rm{JSD}\big(p'_{\rm{mix}},p_g)}+\sqrt{\rm{JSD}\big(p_{\rm{mix}},p'_{\rm{mix}})}\\
		&\qquad \leq o(1)+\sqrt{\delta}\\
		&\qquad \to \sqrt{\delta}.
		\end{align}
		Hence, we have
		\begin{align}
		\limsup_{n\to \infty} \rm{JSD}\big(p_{\rm{mix}},p_g) \leq \delta.
		\end{align} 
		This concludes our proof of Theorem \ref{conv:thm}.
	\end{proof}
	\section{EXPERIMENTS}\label{sec:expe}
	\subsection{$g$ is the standard Gaussian noise vector}
	\subsubsection{Generate new datasets from an old one}
	
	\begin{figure}[!htbp]
		\begin{center}
			\includegraphics[width=70mm, scale=1]{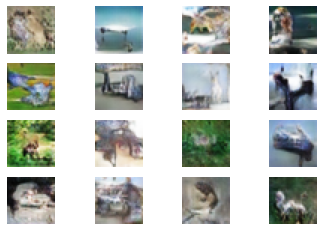}
			\caption{CFAR10, $\delta=0.1$}
			\label{cfar}
		\end{center}
	\end{figure}
	
	\begin{figure}[!htbp]
		\begin{center}
			\includegraphics[width=70mm, scale=1]{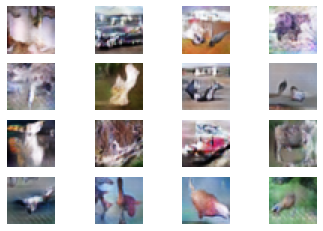}
			\caption{CFAR10, $\delta=0.5$}
			\label{cfar05}
		\end{center}
	\end{figure}

	
	Figures \ref{cfar}-\ref{cfar05} generate new images from CFAR10 for $\delta=0.1$ and $\delta=0.5$, respectively. If we reduce $\delta$, we will obtain images close to CFAR10. It clearly shows the effect of $\delta$ on the generating ability of the GAN. 
	\subsubsection{Generate new datasets from a mixture of old ones}
	\begin{figure}[!htbp]
		\begin{center}
			\includegraphics[width=70mm, scale=1]{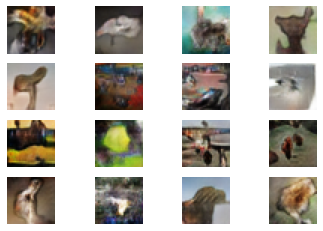}
			\caption{Mixture of CFAR10 and CFAR100, $\delta=0.1$}
			\label{default1}
		\end{center}
	\end{figure}
	

	As the conditional GAN, our algorithm can generate new datasets (i.e., all generated samples have the same distribution) based on mixture of two (or multiple) datasets.  However, since we allows $\delta$ arbitrarily chosen in $[0,1]$, we can generate much more datasets than the conditional GAN.  See Fig.~\ref{default1} for a new dataset which is generated from CFAR10 and CFAR100. 
	\subsection{$g$ is other distribution}
	In this experiment, we use Dirichlet distribution with $\alpha=\bone$. See Fig.~\ref{default10} for a new dataset which is generated from CFAR10. 
	\begin{figure}[!htbp]
		\begin{center}
			\includegraphics[width=70mm, scale=1]{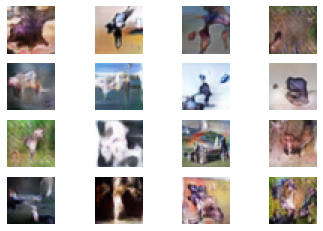}
			\caption{CFAR10, $\delta=0.5$}
			\label{default10}
		\end{center}
	\end{figure}
	\appendix
	\section{Proof of Lemma \ref{lem:a}}\label{lem:aproof}
	We choose a random variable $Z$ with the following distribution:
	\begin{align}
	p_Z(z)=(1-\gamma)\delta(z)+\gamma g(z),
	\end{align} where $g(z)$ is some probability distribution on $\bbR^d$.
	
	Observe that
	\begin{align}
	&2\rm{TV}(p_X,p_Y)\nn\\
	&\qquad =\int_y\bigg|\int_x p_X(x) p_Z(y-x)dx- p_X(y)\bigg|dy\\
	&\qquad =\int_y\bigg|\int_x p_X(x) \big[ (1-\gamma)  \delta(y-x)\nn\\
	&\qquad \qquad + \gamma g(y-x)\big]- p_X(y)\bigg|dy\\
	&\qquad \leq  \int_y\bigg|(1-\gamma) \int_x p_X(x) \delta(y-x)dx- p_X(y)\bigg|dy\nn\\
	&\qquad \qquad +\gamma \int_y \int_x p_X(x)g(y-x)dx dy \label{ab}\\
	&\qquad = \int_y\bigg|(1-\gamma) p_X(y)- p_X(y)\bigg|dy\nn\\
	&\qquad \qquad +\gamma \int_y \int_x p_X(x)g(y-x)dx dy\\
	&\qquad = \gamma \int_y p_Y(y)dy +\gamma \int_y \int_x p_X(x)g(y-x) dx dy\\
	&\qquad =2\gamma,
	\end{align} where \eqref{ab} follows from $|a+b|\leq |a|+|b|$.
	
	Hence, we have
	\begin{align}
	\rm{TV}(p_X,p_Y)\leq \gamma.
	\end{align}
	For $\gamma = 0$, we don't add anything and have $p_X=p_Y$ as expected.
	
	Allowing $\rm{TV}(p_X,p_Y)\not \to 0$ gives us a freedom to choose $g(z)$, the distribution of noise.
	\bibliographystyle{plainnat}
	\bibliography{isitbib}
\end{document}

%% file: preamble.tex
\usepackage[mathscr]{eucal}
\usepackage[cmex10]{amsmath}
\usepackage{epsfig,epsf,psfrag}
\usepackage{amssymb,amsmath,amsthm,amsfonts,latexsym,bm}
\usepackage{amsmath,graphicx,bm,xcolor,url}
\usepackage[caption=false]{subfig} 
\usepackage{array}
\usepackage{verbatim}
\usepackage{bm}
\usepackage{cite}
\usepackage{algpseudocode}
\usepackage{algorithm}
\usepackage{verbatim}
\usepackage{textcomp}
\usepackage{mathrsfs}
\usepackage{multirow}
\usepackage{epstopdf}

\catcode`~=11 \def\UrlSpecials{\do\~{\kern -.15em\lower .7ex\hbox{~}\kern .04em}} \catcode`~=13 

\allowdisplaybreaks[1]
 
\newcommand{\nn}{\nonumber}

\newcommand{\bz}{\mathbf{z}}

\DeclareMathAlphabet{\mathbsf}{OT1}{cmss}{bx}{n}
\DeclareMathAlphabet{\mathssf}{OT1}{cmss}{m}{sl}

\DeclareSymbolFont{bsfletters}{OT1}{cmss}{bx}{n}  
\DeclareSymbolFont{ssfletters}{OT1}{cmss}{m}{n}
\DeclareMathSymbol{\bsfGamma}{0}{bsfletters}{'000}
\DeclareMathSymbol{\ssfGamma}{0}{ssfletters}{'000}
\DeclareMathSymbol{\bsfDelta}{0}{bsfletters}{'001}
\DeclareMathSymbol{\ssfDelta}{0}{ssfletters}{'001}
\DeclareMathSymbol{\bsfTheta}{0}{bsfletters}{'002}
\DeclareMathSymbol{\ssfTheta}{0}{ssfletters}{'002}
\DeclareMathSymbol{\bsfLambda}{0}{bsfletters}{'003}
\DeclareMathSymbol{\ssfLambda}{0}{ssfletters}{'003}
\DeclareMathSymbol{\bsfXi}{0}{bsfletters}{'004}
\DeclareMathSymbol{\ssfXi}{0}{ssfletters}{'004}
\DeclareMathSymbol{\bsfPi}{0}{bsfletters}{'005}
\DeclareMathSymbol{\ssfPi}{0}{ssfletters}{'005}
\DeclareMathSymbol{\bsfSigma}{0}{bsfletters}{'006}
\DeclareMathSymbol{\ssfSigma}{0}{ssfletters}{'006}
\DeclareMathSymbol{\bsfUpsilon}{0}{bsfletters}{'007}
\DeclareMathSymbol{\ssfUpsilon}{0}{ssfletters}{'007}
\DeclareMathSymbol{\bsfPhi}{0}{bsfletters}{'010}
\DeclareMathSymbol{\ssfPhi}{0}{ssfletters}{'010}
\DeclareMathSymbol{\bsfPsi}{0}{bsfletters}{'011}
\DeclareMathSymbol{\ssfPsi}{0}{ssfletters}{'011}
\DeclareMathSymbol{\bsfOmega}{0}{bsfletters}{'012}
\DeclareMathSymbol{\ssfOmega}{0}{ssfletters}{'012}

\newcommand{\tilbx}{\tilde{\bx}}

\newcommand{\bone}{\mathbf{1}}

\newtheorem{theorem}{Theorem} 
\newtheorem{lemma}[theorem]{Lemma}

\newcommand{\qednew}{\nobreak \ifvmode \relax \else
      \ifdim\lastskip<1.5em \hskip-\lastskip
      \hskip1.5em plus0em minus0.5em \fi \nobreak
      \vrule height0.75em width0.5em depth0.25em\fi}